\documentclass[lettersize,journal]{IEEEtran}
\usepackage{amsthm}
\usepackage{amsmath,amsfonts}
\usepackage{braket}
\usepackage{algorithmic}
\usepackage{algorithm}
\usepackage{array}
\usepackage[caption=false,font=normalsize,labelfont=sf,textfont=sf]{subfig}
\usepackage{textcomp}
\usepackage{stfloats}
\usepackage[justification=centering]{caption}
\usepackage{url}
\usepackage{verbatim}
\usepackage{graphicx}
\usepackage{cite}
\usepackage{multirow}
\usepackage{booktabs}
\usepackage{xcolor}
\usepackage{colortbl}
\usepackage{soul}

\usepackage{subcaption}
\theoremstyle{thmstyleone}%
\newtheorem{theorem}{Theorem}
\newtheorem{lemma}{Lemma}

\theoremstyle{thmstyletwo}%

\newcommand{\NAMEA}{QUPID}
\newcommand{\NAMEB}{R-QUPID}

\theoremstyle{thmstylethree}%
\newtheorem{definition}{Definition}%
\renewcommand{\figurename}{Figure}

\begin{document}

\title{\Large{QUPID: A Partitioned Quantum Neural Network for Anomaly Detection in Smart Grid}}

\author{\IEEEauthorblockN{Hoang M. Ngo\IEEEauthorrefmark{2},
Tre’ R. Jeter\IEEEauthorrefmark{2}, 
Jung Taek Seo\IEEEauthorrefmark{3}, and
My T. Thai\IEEEauthorrefmark{2}\IEEEauthorrefmark{1}\thanks{$^*$Corresponding author}}\\
University of Florida, Gainesville, FL, USA\IEEEauthorrefmark{2}\\
Gachon University, Seongnam-daero 1342, Seongnam-si, Republic of Korea\IEEEauthorrefmark{3}\\
Email: \{hoang.ngo, t.jeter\}@ufl.edu, seojt@gachon.ac.kr, and mythai@cise.ufl.edu
}



\maketitle

\begin{abstract}
Smart grid infrastructures have revolutionized energy distribution, but their day-to-day operations require robust anomaly detection methods to counter risks associated with cyber-physical threats and system faults potentially caused by natural disasters, equipment malfunctions, and cyber attacks. Conventional machine learning (ML) models are effective in several domains, yet they struggle to represent the complexities observed in smart grid systems. Furthermore, traditional ML models are highly susceptible to adversarial manipulations, making them increasingly unreliable for real-world deployment. Quantum ML (QML) provides a unique advantage, utilizing quantum-enhanced feature representations to model the intricacies of the high-dimensional nature of smart grid systems while demonstrating greater resilience to adversarial manipulation. In this work, we propose QUPID, a partitioned quantum neural network (PQNN) that outperforms traditional state-of-the-art ML models in anomaly detection. We extend our model to R-QUPID that even maintains its performance when including differential privacy (DP) for enhanced robustness. Moreover, our partitioning framework addresses a significant scalability problem in QML by efficiently distributing computational workloads, making quantum-enhanced anomaly detection practical in large-scale smart grid environments. Our experimental results across various scenarios 
exemplifies the efficacy of QUPID and R-QUPID to significantly improve anomaly detection capabilities and robustness compared to traditional ML approaches.
\end{abstract}


\section{Introduction}\label{sec:intro}
Modern smart grid systems are becoming progressively dependent on automated anomaly detection to ensure their security and reliability. These infrastructures are an interconnected network of sensors, controllers, and distributed energy resources that each produce complex forms of data. However, detecting anomalies in this type of environment is a fundamental challenge due to factors such as energy fluctuations \cite{chen2022physics,liu2023non,mu2023graph,youssef2023adversarial}, renewable energy variability \cite{xie2023bayesian,avramidis2022flexicurity,arpogaus2023short,lu2023risk}, and cyber-physical threats \cite{presekal2024spatio,li2024distributionally,wang2024grid2vec,takiddin2024spatio}. A successful anomaly detection mechanism depends on two essential aspects: performance and robustness. High performance ensures accurate identification of anomalous behavior in real-time while robustness guarantees resilience against adversarial manipulation. Without access to models that can excel in both regards, smart grid systems remain vulnerable to costly disruptions, security breaches, and operational failures.

Despite the increased usage of machine learning (ML) models for anomaly detection in smart grid systems, existing approaches are limited in both performance and adversarial robustness. Traditional ML models struggle to effectively interpret the high-dimensional and non-linear characteristics of smart grid data because they are designed with the assumption that they will operate on structured data and in structured environments \cite{pmlr-v48-daib16}. Beyond performance, another growing concern is their susceptibility to adversarial manipulation. Disturbances introduced in the input data can manipulate model outputs, leading to unreliable predictions \cite{khan2024adaptedge}. In an effort to mitigate these threats, traditional ML approaches often introduce differential privacy (DP) via calibrated Gaussian or Laplacian noise \cite{wang2022privacy,lu2023privacy,huang2022dpwgan}. However, DP protocols have recently become significantly ineffective in traditional ML models as modern day adversaries have learned to adapt their attacks to bypass them \cite{giraldo2020adversarial,li2024analyzing,li2023fine}. As a result, many traditional ML models miss anomalies and incorrectly flag normal operations, making them extremely unreliable. 

To address the limitations of traditional ML models for anomaly detection in smart grid environments, we propose QUPID as a \underline{\textbf{Q}}uant\underline{\textbf{U}}m \underline{\textbf{P}}art\underline{\textbf{I}}tioned neural network for anomaly \underline{\textbf{D}}etection. Smart grid data can come in the form of real and complex values. However, traditional ML models are only capable of interpreting real-value data, making them ineffective in realistic cases. Because of properties innate to quantum computing, QUPID can process complex values, resulting in more precise and enhanced feature representations in relationships produced by smart grid data. Additionally, QUPID introduces a partitioning scheme designed to address a key scalability challenge in QML. Another distinct advantage of quantum computing is the presence of inherent quantum noise, an effect that cannot be reliably simulated by classical systems. Although conventional models add classical noise to achieve DP guarantees, they have become 
increasingly vulnerable to sophisticated attacks, whereas quantum-induced noise can act as a privacy amplifier, significantly strengthening the guarantees of the classical DP mechanism. This amplified robustness makes traditional adversarial attacks far less effective and impractical to simulate. Ultimately, QUPID yields state-of-the-art anomaly detection performance, while its robust variant, R-QUPID, integrates quantum-induced noise for enhanced privacy, offering certified robustness guarantees against adversarial threats.

\noindent \textbf{Contributions.} Our main contributions are as follows:
\begin{itemize}
    \item We develop QUPID, a QML model capable of encoding and processing complex value data, addressing the performance limitations of traditional ML approaches.
    \item  We introduce a partitioning scheme in QUPID that addresses the scalability limitations native to QML by reducing the number of qubits required while improving training efficiency.
    \item  To our knowledge, 
    our work is the first to formally analyze how quantum noise can amplify the privacy guarantees of a classical noise mechanism. We provide theoretical proofs of these amplification effects and demonstrate their role in enhancing certifiable robustness to adversarial manipulation in a hybrid quantum-classical model.
    \item  QUPID and R-QUPID consistently outperform five state-of-the-art baselines including four classical deep learning (DL) models and one hybrid quantum-classical model 
    across 15 scenarios of the ICS dataset~\mbox{\cite{Morris_ICSDataset}} and 7 metrics, exemplifying superior anomaly detection performance and adversarial robustness.

\end{itemize}

\noindent\textbf{Organization.} The rest of the manuscript is structured as follows: Section~\ref{sec:related} spans relevant works related to anomaly classification in smart grid systems, QML usage for classification, and a robustness comparison of traditional ML models to QML models. Section~\ref{sec:background} gives background information about anomalies seen in smart grid environments and quantum neural networks (QNNs). We formally describe the model architecture of QUPID and its certified robustness guarantees via quantum noise in Section~\ref{sec:pqnn}. Section~\ref{sec:exp} outlines our experimental analysis. Section~\ref{sec:con} provides closing remarks and concludes the paper.
\section{Related Works}\label{sec:related}
\noindent \textbf{Anomaly Classification in Smart Grids.} Anomaly detection in smart grid infrastructures is critical in ensuring operational reliability and security. Conventional methods like supervised statistical thresholding and unsupervised rule-based heuristics are well-established methods but often struggle to adapt to the complex, high-dimensional nature of smart grid data \cite{komadina2024comparing}. Statistical thresholding methods using Receiver Operating Characteristics (ROC) curves \cite{reiss2023mean}, precision recall curves \cite{tatbul2018precision,brodersen2010binormal}, or equal error rate (EER) calculations \cite{ullah2018anomalous}, optimize particular evaluation metrics and trade-offs, but heavily depend on the ability of the underlying model to produce distinct anomaly scores. If the model itself struggles to differentiate between normal and anomalous patterns, an optimal threshold alone will not guarantee high performance. Besides, supervised methods such as statistical thresholding require labeled data which is unrealistic in most cases, and their static nature limits adaptability to ever changing data distributions. Unsupervised methods aim to address these issues, but suffer from their own limitations. Many unsupervised approaches require continuous calibration to adapt to dynamic thresholds \cite{ahmad2017unsupervised}, while many parametric methods require data-specific tuning \cite{poon2021unsupervised}. Clustering and density-based methods are well-established approaches in unsupervised ML, but they rely heavily on assumptions about data structure that may be unrepresentative of real-world smart grid environments \cite{zhao2023searching}.\\

\noindent \textbf{Quantum ML for Classification.} QML has materialized as an auspicious substitute for traditional ML approaches, as it takes advantage of quantum principles to enhance computational efficiency and data representation. QML models are able to process complex-valued quantum states, unlike traditional models, that yield more useful feature embeddings and improved pattern recognition. Quantum enhanced classifiers such as variational quantum circuits (VQCs) \cite{wu2023more,li2021vsql,chen2024novel} and quantum kernels (QKs) \cite{alvarez2025benchmarking} have been implemented for many classification tasks in domains ranging from image recognition to cybersecurity \cite{senokosov2024quantum,saggi2024mqml,DBLP:journals/iotj/SatpathyVBAMF24,innan2024financial}. 
 
Despite these recent advancements, most quantum neural networks (QNNs) are limited to binary classification tasks due to resource constraints and/or intensive classical post-processing. Addressing this limitation, MORE \cite{wu2023more} is presented as a VQC-based multi-classifier leveraging quantum state tomography to extract full information with a single readout qubit. With a VQC clustering approach, MORE generates quantum labels and uses supervised learning techniques to match inputs to their respective labels. Comparably, variational shadow quantum learning (VSQL) \cite{li2021vsql} is a hybrid quantum-classical approach utilizing compressed representations of quantum data (classical shadows) for feature extraction. Variational shadow quantum circuits execute convolution-like operations while a fully connected classical neural network is used for the final classification task. The VSQL approach significantly reduces the number of parameters needed for training and alleviates the vanishing gradient problem seen in QML (Barren Plateau \cite{mcclean2018barren}). 

By simply removing a global pooling layer typically seen in classical methods, \cite{chen2024novel} propose hybrid VQCs and algorithms that reduce computational complexity and information loss, leading to higher performance in image classification tasks. Exploring the practicality of QKs, \cite{alvarez2025benchmarking} presents a benchmarking investigation comparing quantum kernel estimation and quantum kernel training with classical ML approaches. QK results are typically better on ad-hoc datasets, but vary on standard datasets. Among these observations is that increasing computational cost is not directly proportional to performance either, suggesting that optimization-based approaches should be implemented in QML frameworks. 

Despite these improvements, current QML approaches are still limited in terms of performance, scalability, and dataset complexity. The majority of comparisons rely on simple traditional model architectures such as one-layer neural networks, ResNet, and MaxViT on proportionally small-scale datasets. Our approach, QUPID, is tailored for multi-classification of complex anomalies in smart grid environments and is comparable to current state-of-the-art classical classification methods. QUPID also addresses the scalability issue by using a partitioning scheme to manage workloads.\\

\noindent \textbf{Adversarial Robustness of ML and QML Models.} The robustness of conventional ML models has been thoroughly challenged via adversarial threats where perturbations are introduced to cause misclassifications. In cases of anomaly detection, misclassifications due to adversarial threats can be the cause of fatal accidents. Traditional ML approaches typically rely on adversarial training and differential privacy (DP) techniques via Gaussian or Laplacian noise injection. However, adversarial attacks have evolved enough to exploit the learned noise distributions, making these defenses ineffective \cite{giraldo2020adversarial}. On the other hand, QML models innately introduce quantum noise due to their intrinsic operations, which is thought to provide instinctive defenses against adversarial threats \cite{wendlinger2024comparative,watkins2023quantum,west2023benchmarking}. Current research has committed to investigating the impact of quantum noise on model robustness to reveal if QML models can withstand adversarial threats that commonly succeed in traditional ML models. 

In particular, VQCs exhibit higher resilience against classical adversarial threats. Although adversarial training significantly enhances traditional ML model robustness, it has proven to have minimal impact on QML models \cite{west2023benchmarking}. While the full extent of quantum-specific adversarial threats is underexplored, the few efforts in research have proven to degrade the performance of QML models \cite{wendlinger2024comparative}. Moreover, as it pertains to QML models, the incorporation of differential privacy (DP) techniques as in DP-VQCs does not significantly degrade model performance, a trend typically observed in classical ML models \cite{watkins2023quantum}. Recent works \mbox{\cite{Du2021Depolarizing,Hirche2023}} have explored quantum-induced noise as a natural source of randomness to achieve DP guarantees, thereby enhancing the adversarial robustness of purely QML models that operate exclusively on quantum inputs. However, existing literature treats either classical noise or quantum noise as the sole mechanism for achieving privacy. To date, no work has investigated the combined use of both noise types as a unified privacy-preserving strategy for 
hybrid QML models that integrates both classical and quantum components. This gap represents an important opportunity for advancing the privacy and robustness of practical QML systems.
\section{Background}\label{sec:background}
\subsection{Taxonomy of Anomalies in Smart Grids}
In an effort to provide a clear understanding of anomalies in smart grids in a nonexhaustive manner, we categorize them into three main types: normal operations, natural events, and attack scenarios. Normal operations refer to expected energy fluctuations and load variations that naturally occur in smart grid systems. Natural events can be sub-categorized into short-circuit faults and line maintenance. Short-circuit faults can occur at different locations along a transmission line. Line maintenance includes the disabling of relays to preserve the overall system. 

On the other hand, attack scenarios span numerous potential cyber threats that can disrupt grid stability, including remote tripping command injections, relay setting changes, and data injection attacks. Remote tripping attacks occur when an adversary has bypassed outside defenses and gains access to the system. From there, the attacker sends malicious commands to relays that force breakers to open. Relay setting attacks occur when an adversary reconfigures the distance protection scheme of a relay, inhibiting it from recognizing valid system faults and valid incoming commands. Data injection attacks falsify voltage, current, and sequence measurements to imitate fault signals in the system. These false signals can deceive system operators and potentially cause major outages.

\subsection{Quantum Computing}
\noindent \textbf{Applications in Smart Grids}

Grid systems are rapidly evolving due to the growing integration of intermittent renewable energy sources (such as solar and wind), electric vehicles, and distributed energy resources. This leads to optimization and forecasting problems with enormous search spaces and complex, non-linear relationships. The role of quantum computing in smart grids is not to replace classical computing across all tasks, but rather to address a specific subset of problems that are becoming intractable even for the most advanced classical supercomputers. For instance, QML is anticipated to generate and uncover patterns that are infeasible for classical models to detect~\mbox{\cite{Schuld2019}}. This allows QML to identify complex patterns in grid data that are invisible to classical algorithms, leading to more accurate load forecasting or anomaly detection.

Access to quantum computing has become increasingly feasible. Major technology companies such as IBM and Google now offer quantum cloud computing platforms that enable broad and relatively easy access to quantum processors. In addition, the integration of Quantum Processing Units (QPUs) into supercomputing architectures is underway, with companies like NVIDIA actively working toward hybrid systems that combine CPUs, GPUs, and QPUs. This paves the way for a near future in which quantum-accelerated supercomputers become widely available and more convenient for end-users.

In the current Noisy Intermediate-Scale Quantum (NISQ) era, we are working with quantum devices that include a few hundred to over a thousand qubits, with gate fidelities exceeding 99.9\%~\mbox{\cite{Lin2025}}. Although full-scale fault-tolerant quantum computing remains a long-term goal, near-term deployment of quantum computing is becoming practical for localized applications—for example, within a single microgrid or a small cluster of distributed energy resources (DERs). Given the increasing prevalence of distributed systems, QPU-integrated supercomputers are well-suited to monitor, analyze, and control edge components in the power grid where fast, high-dimensional decision-making is required.

\noindent \textbf{Fundamental Concepts in Quantum Computing.} Quantum computing operates on quantum bits (\emph{qubits}), which, unlike classical bits, can exist in a superposition of 0 and 1. An $n$-qubit system is represented as a superposition of $2^n$ basis states. For example, a 2-qubit system is expressed in the form:
\begin{equation}
    |\psi\rangle = \alpha_{00} |00\rangle + \alpha_{01} |01\rangle + \alpha_{10} |10\rangle + \alpha_{11} |11\rangle
\end{equation}

Such states, known as \emph{pure states}, are fully coherent quantum states that can be described by ket vectors. However, real-world quantum systems often exist in \emph{mixed states}, which are statistical ensembles of pure states. These ensembles are represented by a density matrix:
\begin{equation}
    \rho = \sum_i p_i |\psi_i\rangle\langle\psi_i|
\end{equation}
where $|\psi_i\rangle$ occurs with probability $p_i$. 

The evolution of a quantum state, pure or mixed, is governed by \emph{quantum channels}. The transformation of a density matrix $\rho$ through a quantum channel $\mathcal{U}$ is given by $\mathcal{U}(\rho)$.

In quantum mechanics, a measurement is described by a set of operators $\{M_m\}$, where the probability of obtaining outcome $m$ is given by:
\begin{equation}
    p(m) = \text{tr}(M_m^\dagger M_m \rho).
\end{equation}

\noindent \textbf{Quantum Neural Networks.} As a class of QML models, quantum neural networks (QNNs) utilize parameterized quantum circuits to process classical or quantum data. In this work, we primarily focus QNNs operating on classical data. In a supervised learning setting, a QNN approximates an unknown function \( K: \mathbb{X} \to \mathbb{Y} \) by training on a dataset \( S = \{(x_i, y_i)\}_{i=1}^{N} \), where \( x_i \) represents input data and \( y_i \) is the corresponding label. The function \( K \) can be decomposed into three main components: an encoder \(U_{enc}\), a parameterized ansatz \(U_{ans}\), and a measurement $M$ where:
\begin{itemize}
    \item \textbf{Encoder} \( U_{enc} \) maps the classical input data \( x \) to a quantum state \( |x\rangle \) in the Hilbert space \( \mathcal{H} \). This is achieved by applying \( U_{enc} \) parameterized by $x$ to an initial reference state:
    \begin{equation}
        |x\rangle = U_{enc}(x) |0\rangle^{\otimes n}.
    \end{equation}
    \item \textbf{Parameterized Ansatz} \( U_{ans}(\theta) \) is a unitary transformation governed by a set of trainable parameters \( \theta \). The ansatz defines the QNNs hypothesis space and evolves the encoded state:
    \begin{equation}
        |x'\rangle = U_{ans}(\theta) U_{enc}(x) |0\rangle^{\otimes n}.
    \end{equation}
    \item \textbf{Measurement} \( M = \{M_i\}\) maps the output quantum state to classical values. This is achieved by measuring the expectation value of a Hermitian observable \( M_i \) corresponding to the target label:
    \begin{equation}
        \ell_i(\theta; y_i) = \text{Tr}(M_i^\dagger M_i \ket{x'}\bra{x'}),
    \end{equation}
\end{itemize}

Training a QNN involves optimizing the parameters \( \theta \) to minimize a loss function that quantifies the discrepancy between predictions and true labels. Given a differentiable loss function \( f(\cdot) \) (e.g., Mean Squared Error (MSE) or Cross-Entropy (CE)), the objective is to minimize:
\begin{equation}
    L(\theta) = \sum_{i=1}^{N} f(\ell_i(\theta; y_i), y_i).
\end{equation}

Optimization is typically performed using classical gradient-based algorithms, such as gradient descent or its variants, yielding the optimal parameters:
\begin{equation}
    \theta^* = \arg \min_{\theta} L(\theta).
\end{equation}

\noindent \textbf{Adversarial Robustness.} Adversarial robustness refers to a model’s ability to maintain consistent predictions despite small, carefully crafted perturbations to the input data, called adversarial examples. Formally, given a model \( f: \mathbb{X} \to \mathbb{Y} \), where \( y = \{y_1,y_2,\dots, y_k\} \in \mathbb{Y}\) represents the output label distribution, $f$ is considered robust if the predicted label remains unchanged when small perturbations \( \alpha \) are added to the input \( x \). More formally, this can be expressed as:

\[
\max_{i \in [1,k]} [f(x)]_i = \max_{i \in [1,k]} [f(x + \alpha)]_i, \quad \forall \alpha \in B_p(L),
\]

\noindent where \( B_p(L) \) represents the \( p \)-norm ball of radius \( L \), that restricts the perturbation size to \( \|\alpha\|_p \leq L \).

Recently, researchers have explored using Differential Privacy (DP) as a solution to enhancing model robustness. DP is a privacy-preserving technique that introduces calibrated randomness to a model's output, ensuring that small changes in the input data do not significantly affect final predictions. More formally, $\varepsilon$-DP can be defined as:
\begin{definition}{$\varepsilon, \delta$-DP:} A randomized algorithm $\mathcal{M}$ satisfies $\varepsilon, \delta$-DP, if for any subset $\mathcal{S} \subseteq \textup{Range}(\mathcal{M})$ and any two adjacent datasets $D$ and $D'$, the following condition holds:
\begin{equation}
    Pr[\mathcal{M}(D) \in \mathcal{S}] \leq e^{\varepsilon} Pr[\mathcal{M}(D') \in \mathcal{S}] +\delta,
\end{equation}
where $\varepsilon$ is a privacy budget that controls the trade-off between privacy and model performance. Smaller values of $\varepsilon$ yield stronger privacy guarantees while introducing larger amounts of noise that may alter the predictions of the model, thus affecting the performance.
\label{dp} 
\end{definition}

The connection between adversarial robustness and DP is formalized through the idea that injecting DP noise to a model improves prediction stability. Specifically, consider a model \( f \) that satisfies \( (\epsilon, \delta) \)-DP with respect to a \( p \)-norm metric. The model \( f \) is robust to adversarial perturbations \( \alpha \) of size \( \|\alpha\|_p \leq 1 \) if the following condition holds:

\[
E([f(x)]_k) > e^\epsilon \max_{i: i \neq k} E([f(x)]_i) + (1 + e^\epsilon) \delta, \hspace*{2pt} \text{for some } k \in K,
\]

\noindent where \( E([f(x)]_k) \) represents the expected confidence score of the predicted label \( k \), and \( E([f(x)]_i) \) is the expected confidence score for other labels. 

This condition guarantees adversarial robustness for the set of inputs that satisfy the above inequality. A stronger DP guarantee (i.e., smaller \( \epsilon \) and \( \delta \)) leads to greater robustness across a broader range of inputs. In this work, we investigate how quantum noise amplifies the effects of DP, thereby further enhancing the robustness of the model.
\section{QUPID - A Partitioned Quantum Neural Network for Anomaly Classification}\label{sec:pqnn}
In this section, we describe the architecture of QUPID, our partitioned quantum neural network (PQNN) for anomaly classification. Then, we demonstrate QUPIDs robustness under the influence of quantum noise.

\subsection{Complex Input Data}
In this work, we consider a grid system whose state is measured by \( n \) phasor measurement units (PMUs). The \( i \)-th PMU provides comprehensive electrical measurements, including three-phase voltage magnitudes and angles, represented as  
\[
[V^{(i)}_a, V^{(i)}_b, V^{(i)}_c, \alpha^{(i)}_a, \alpha^{(i)}_b, \alpha^{(i)}_c],
\]
as well as symmetrical component voltage magnitudes and angles for positive, negative, and zero sequence components:  
\[
[V^{(i)}_+, V^{(i)}_-, V^{(i)}_z, \alpha^{(i)}_+, \alpha^{(i)}_-, \alpha^{(i)}_z].
\]
Similarly, the PMU records the three-phase current magnitudes and angles:  
\[
[I^{(i)}_a, I^{(i)}_b, I^{(i)}_c, \beta^{(i)}_a, \beta^{(i)}_b, \beta^{(i)}_c],
\]
along with the symmetrical component current magnitudes and angles:  
\[
[I^{(i)}_+, I^{(i)}_-, I^{(i)}_z, \beta^{(i)}_+, \beta^{(i)}_-, \beta^{(i)}_z].
\]

To effectively capture the complex relationships between magnitudes and angles, we represent the PMU measurements as complex phasors. Specifically, the measurement vector for the \( i \)-th PMU is expressed as:

\[
x^{(i)} = 
\begin{bmatrix}
V^{(i)}_a e^{j\alpha^{(i)}_a}, V^{(i)}_b e^{j\alpha^{(i)}_b}, V^{(i)}_c e^{j\alpha^{(i)}_c}, \\
V^{(i)}_+ e^{j\alpha^{(i)}_+}, V^{(i)}_- e^{j\alpha^{(i)}_-}, V^{(i)}_z e^{j\alpha^{(i)}_z}, \\
I^{(i)}_a e^{j\beta^{(i)}_a}, I^{(i)}_b e^{j\beta^{(i)}_b}, I^{(i)}_c e^{j\beta^{(i)}_c}, \\
I^{(i)}_+ e^{j\beta^{(i)}_+}, I^{(i)}_- e^{j\beta^{(i)}_-}, I^{(i)}_z e^{j\beta^{(i)}_z}
\end{bmatrix}
\]

\begin{figure*}[t]
    \centering\includegraphics[width=1\textwidth, trim={0.15cm 0.1cm 0.1cm 0.55cm},clip=True]{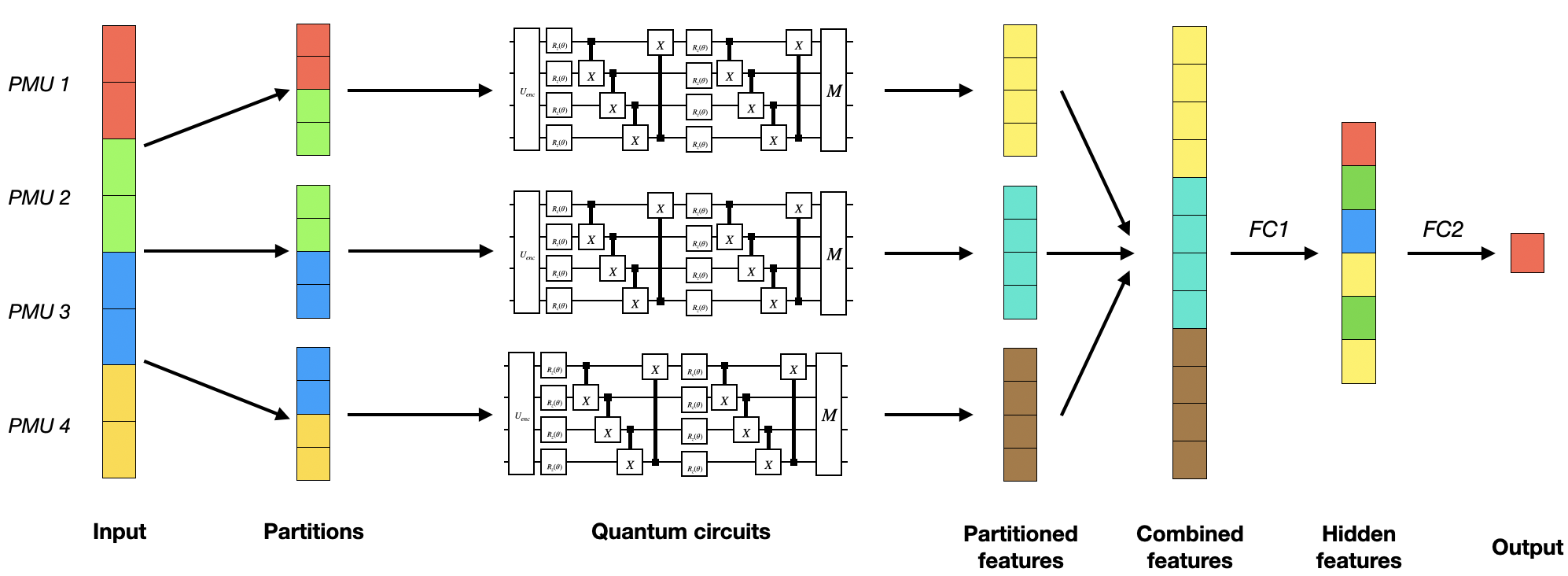}
    \caption{Architectural Design of QUPID.}
    \label{fig:QNN}
    
\end{figure*}

\subsection{Architectural Design of QUPID}
QUPID is composed of two key features that enhance representation capabilities and scalability. First, the model partitions input data by grouping information from neighboring phasor measurement units (PMUs). This localized processing improves sensitivity to subtle fluctuations in voltage and current, which typically impact a confined region rather than an entire grid system. Additionally, each partition is processed by a small parameterized quantum circuit (PQC), addressing the scalability challenges typically observed in QML and, more generally, in quantum computing. Second, QUPID encodes voltage and current as complex numbers in the quantum space, allowing the model to better capture the intricate relationships in the data. This encoding scheme leverages the inherent mathematical structure of quantum systems to improve overall expressiveness of learned representations. 

QUPID is illustrated in Figure~\ref{fig:QNN}. Given the grid state \(\mathbf{x} = [x^{(1)}, \dots, x^{(n)}]\), where each \( x^{(i)} = [x^{(i)}_1, \dots, x^{(i)}_m] \) represents complex phasor information from the \( i \)-th PMU, the input \( \mathbf{x} \) is partitioned into \( K \) subsets: \( \mathbf{x}^{(1)}, \dots, \mathbf{x}^{(K)} \). Without loss of generality, we assume each partition is composed of an equal number of PMUs. Specifically, each partition \( \mathbf{x}^{(k)} = [x^{(l_1)}, \dots, x^{(l_{t})}] \) consists of data from PMUs that are spatially or topologically connected, ensuring that the \( l_i \)-th PMU is a neighbor of the \( l_1 \)-th PMU. The notion of \textit{neighboring} can be defined on the basis of physical distance or grid topology (e.g., by the number of hops between PMUs in the system).

Each partition \(\mathbf{x}^{(k)}\) is processed by a PQC, denoted as \( C^{(k)} \), with \(q = \log_2(t \times m) \) qubits. In total, QUPID requires $\mathcal{O}(K\log_2(t \times m))$ qubits. As mentioned in Section~\ref{sec:background}, \( C^{(k)} \) consists of three main components: an encoder \( U^{(k)}_{\text{enc}} \), a parameterized ansatz \( U^{(k)}_{\text{ans}}(\theta_k) \), and a measurement operator \( M^{(k)} \). To encode complex numbers as quantum state amplitudes, we apply an amplitude encoding scheme:

\[
|\mathbf{x}^{(k)}\rangle = U^{(k)}_{\text{enc}} |0\rangle^{\otimes q} = \sum_{i=1}^{t} \sum_{j=1}^{m} d^{(l_i)}_j |l_i \times m + j\rangle
\]

\noindent where the encoding gates used to construct this encoder are based on~\cite{Iten2016}. The encoded quantum state is then transformed through the ansatz:

\[
|\mathbf{x'}^{(k)}\rangle = U^{(k)}_{\text{ans}} (\theta_k)|\mathbf{x}^{(k)}\rangle
\]

The output of each PQC is a latent feature vector \( h^{(k)} \in \mathbb{R}^{\log_2 q} \), where each element is computed as:

\[
h^{(k)}_i = \text{Tr}(M^{(k)\dagger}_i M^{(k)}_i |\mathbf{x'}^{(k)}\rangle \langle \mathbf{x'}^{(k)}|)
\]

Next, we concatenate the outputs from all \( K \) quantum circuits to form the final latent feature representation:

\[
\mathbf{h} = [h^{(1)}, \dots, h^{(K)}]
\]

The latent feature vector \( \mathbf{h} \) is then passed through a fully connected function $F$, consisting of two dense layers with ReLU activations, parameterized by \( \theta^{(fc)}_1 \) and \( \theta^{(fc)}_2 \): 

\[
z = F(\mathbf{h}) = \text{ReLU}(\theta^{(fc)}_2 \text{ReLU}(\theta^{(fc)}_1 \mathbf{h}))
\]

To better illustrate QUPIDs overall flow, we define its mechanism $\mathcal{M}_{\text{QUPID}}$, in terms of PQCs \( C^{(1)}, \dots, C^{(K)} \) and \( F \) as follows:  
\[
\mathcal{M}_{\text{QUPID}} = F \circ \left( C^{(1)} \oplus C^{(2)} \oplus \dots \oplus C^{(K)} \right)
\]

\noindent where \( \oplus \) denotes the parallel application of each PQC to its corresponding input partition, while \( \circ \) represents sequential mechanism composition.

Finally, QUPID outputs the predicted class probabilities $\hat{y}$ using a softmax layer applied to $z$.

\[
\hat{y} = \text{softmax}(\mathcal{M}_{\text{QUPID}}(\mathbf{x})) = \text{softmax}(z)
\]

The model is trained by optimizing the CE loss function:
\[
\mathcal{L} = - \sum_{i} y_i \log(\hat{y}_i)
\]
where \( y_i \) represents the true class label and \( \hat{y}_i \) is the predicted probability for class \( i \). Model parameters are updated using gradient descent based on the computed loss.

\subsection{Adversarial Robustness via Quantum Noise}
In traditional ML models, differential privacy (DP) is typically achieved by injecting Gaussian or Laplacian \textit{classical} noise into the data~\cite{Lcuyer2018CertifiedRT}. DP ensures that predictions remain differentially private with respect to input features, thus enhancing robustness against adversarial examples. In this section, we introduce a provably robust mechanism for \NAMEA, named \NAMEB, that integrates classical and quantum noise as additional layers to enhance privacy and robustness. Furthermore, we formally derive the privacy guarantees of \NAMEB and prove its certifiable robustness by analyzing the combined influence of classical and quantum noise.

As shown in~\cite{Lcuyer2018CertifiedRT}, a deep learning model can achieve a privacy guarantee of \((\epsilon, \delta)\)-DP by adding classical noise directly to the input. The DP parameters \(\epsilon\) and \(\delta\) depend on the characteristics of the noise. Specifically, to achieve \((\epsilon, \delta)\)-DP, Gaussian noise with a mean of zero and a standard deviation of  
\(
\sigma = \sqrt{2\ln\left(\frac{1.25}{\delta}\right)} \frac{\Delta L}{\epsilon}
\)  
is required, where \(\Delta\) represents the input sensitivity, and \(L\) denotes the attack bound. For later analysis, we denote the addition of classical noise as $A$. 

In contrast to classical noise, quantum noise naturally occurs after the application of each quantum gate. One widely studied type of quantum noise is depolarizing noise. As shown in~\cite{Du2021Depolarizing}, the cumulative effect of depolarizing noise channels applied after every gate is equivalent to that of a single noise channel, where the overall noise strength accumulates from the individual noise channels. Based on this observation, we construct a noisy counterpart of the circuits used in \NAMEA\ by incorporating a depolarizing channel immediately after the parameterized ansatz. We denote this noisy counterpart as $\tilde{C}$. Our objective is to analyze the privacy guarantee of the combined noise mechanism $\tilde{C} \circ A$. We first establish the condition for privacy amplification by a quantum mechanism as follows:

\begin{lemma}
    Let $A$ be an $(\epsilon, \delta)$-DP mechanism. For a given quantum mechanism $\tilde{C}:\mathbb{X} \rightarrow \mathcal{P}(\mathbb{Y})$, we define $\tilde{C}$ as $\lambda-$\emph{Dobrushin} if: 
    \begin{equation}
        \sup_{x,x'} D_1 (\tilde{C}(x)||\tilde{C}(x')) \leq \lambda
    \end{equation} 
    The mechanism $\tilde{C} \circ A$ satisfies $(\epsilon, \delta')$-DP with $\delta' = \delta \lambda$ if $\tilde{C}$ is $\lambda-$\emph{Dobrushin}. 
    \label{lemma:Dobrushin}
\end{lemma}

\begin{proof}
    From \cite{Balle2019PrivacyAmplification}, this lemma holds if and only if $\tilde{C}$ is a Markovian process. Additionally, as stated in~\cite{Brand2024Markovian}, noisy quantum mechanisms can be represented as Markovian processes. Thus, the claim is valid.
\end{proof}

The degree of privacy amplification depends on the upper bound of the Total Variation (TV) Distance, denoted as \( D_1 \), which is a special case of the hockey-stick divergence \( D_\gamma \), between the distributions of \( \tilde{C}(x) \) and \( \tilde{C}(x') \), where \( x \sim x' \). To determine the upper bound \( \lambda \), we establish the equivalence between the hockey-stick divergence and the quantum hockey-stick divergence, as follows:

\begin{lemma}
    Given a measurement \( M = \{M_i\} \) with \( \sum_i M_i = I \), and two quantum states \( \rho \) and \( \rho' \), the hockey-stick divergence of the probability distributions resulting from measuring \( \rho \) and \( \rho' \) under \( M \) is upper-bounded by the quantum hockey-stick divergence between \( \rho \) and \( \rho' \).
    \label{lemma:QHSD_equiv_HSD}
\end{lemma}

\begin{proof}
The quantum hockey-stick divergence is defined as:
\[
D^{(q)}_\gamma(\rho \parallel \rho') = \text{Tr}\big[(\rho - \gamma \rho')_+\big],
\]
where \( (\rho - \gamma \rho')_+ \) denotes the positive part of the eigendecomposition of the Hermitian matrix \( \rho - \gamma \rho' \).

Let the measurement \( M = \{M_i\} \) be applied to \( \rho \) and \( \rho' \). The resulting probability distributions are:
\[
P(i) = \text{Tr}(M_i \rho), \quad Q(i) = \text{Tr}(M_i \rho').
\]

The hockey-stick divergence for the classical distributions \( P \) and \( Q \) are defined as:
\[
D_\gamma(P \parallel Q) = \sum_i \big[P(i) - \gamma Q(i)\big]_+,
\]
where \( [x]_+ = \max(x, 0) \).

To show the equivalence of the quantum and classical hockey-stick divergences, we have:
\begin{align*}
    D_\gamma(P \parallel Q) &= \sum_i \big[P(i) - \gamma Q(i)\big]_+ \\
    &= \sum_i \max(0, \text{Tr} (M_i \rho) - \gamma \text{Tr}(M_i\rho')) \\
    &= \sum_i  [\text{Tr}(M_i (\rho - \gamma \rho'))]_+
\end{align*}

Let the operator $A = \rho - \gamma \rho'$, and let its decomposition into positive and negative parts be $A = A_+ - A_-$, where $A_+$ and $A_-$ are positive semidefinite operators. Since $\text{Tr}(M_i A_-) \geq 0$ (as both $M_i$ and $A_-$ are positive semidefinite), we have $\text{Tr}(M_i A) = \text{Tr}(M_i A_+) - \text{Tr}(M_i A_-) \le \text{Tr}(M_i A_+)$. Because $\text{Tr}(M_i A_+) \geq 0$, we have:
$$\big[\text{Tr}(M_i A)\big]_+ \le \big[\text{Tr}(M_i A_+)\big]_+ = \text{Tr}(M_i A_+)$$

Summing over all measurement outcomes $i$:
\begin{align*}
    D_\gamma(P \parallel Q) &\le \sum_i \text{Tr}(M_i A_+) \\
    &= \text{Tr}\left[\left(\sum_i M_i\right) A_+\right] \\
    &= \text{Tr}[I \cdot A_+] \\
    &= \text{Tr}[A_+] = D^{(q)}_\gamma(\rho \parallel \rho').
\end{align*}

Thus, the hockey-stick divergence of the probability distributions \( P \) and \( Q \) resulting from the measurement \( M \) is upper-bounded by the quantum hockey-stick divergence between \( \rho \) and \( \rho' \). 

\end{proof}

Additionally, we can derive the upper bound on the quantum hockey-stick divergence for a quantum channel affected by depolarizing noise applied to two states, as follows:

\begin{lemma}
    Given a depolarizing channel defined as $\mathcal{E}_p(\rho) = (1-p) \rho + p\frac{I}{d}$, for $p \in [0,1]$ and $\gamma \geq 1$, we have:
    \[
        D^{(q)}_{\gamma} (\mathcal{E}_p(\rho)\parallel \mathcal{E}_p(\rho') ) \leq \max\{0, (1-\gamma)\frac{p}{d}+(1-p) D^{(q)}_{\gamma}(\rho \parallel \rho')\}
    \]
    \label{lemma:QHSD_depolarizing}
\end{lemma}
\vspace*{-0.75cm}
\begin{proof}
    We have:
    \begin{align*}
    &D^{(q)}_{\gamma} (\mathcal{E}_p(\rho)\parallel \mathcal{E}_p(\rho') )\\
    &= \text{Tr}[ (1-\gamma)\frac{p}{d} I + (1-p)(\rho - \gamma \rho') ]_+ \\
    &=  \text{Tr} [P^+[ (1-\gamma)\frac{p}{d} I + (1-p)(\rho - \gamma \rho') ]] \\&\text{($P^+$ is the projector onto the positive subspace)}\\
    &= [ (1-\gamma)\frac{p}{d}]\text{Tr}P^+ + (1-p)\text{Tr}[P^+(\rho - \gamma \rho') ] \\
    &\leq (1-\gamma)\frac{p}{d} + (1-p)D^{(q)}_{\gamma}(\rho \parallel \rho') \\ &\text{(since $D^{(q)}_{\gamma} (\mathcal{E}_p(\rho)\parallel \mathcal{E}_p(\rho') ) > 0$, it follows that Tr$P^+ \geq 1$)} 
\end{align*}
\end{proof}

By combining Lemmas~\ref{lemma:Dobrushin}, \ref{lemma:QHSD_equiv_HSD}, and \ref{lemma:QHSD_depolarizing}, we can derive the $(\epsilon, \delta)-$DP mechanism of $\tilde{C} \circ A$ acting on the input $\mathbf{x}$ as follows:

\begin{lemma}
    [Privacy Amplification via Depolarizing Noise] Let $A$ be an $(\epsilon, \delta)-$DP mechanism. We define a noisy quantum circuit $\tilde{C}$ including the encoder $U_{enc}$, the parameterized ansatz $U_{ans}(\theta)$ followed by the depolarizing channel $\mathcal{E}_p$ and measurement $M$. Given the minimum dot product 
    \begin{equation}
        \phi = \min_{x,x' \in \mathbb{C}^{2^q}} |\langle x|x' \rangle|^2,
    \end{equation}
    the mechanism $\tilde{C} \circ A$ satisfies $(\epsilon, (1-p)\sqrt{1-\phi}\delta)$-DP.
    \label{lemma:privacy_amplification}
\end{lemma}

\begin{proof}
    We have:
    \begin{align*}
        &\sup_{x,x'} D_1(\tilde{C}(x)\parallel \tilde{C}(x'))\\ &\leq \sup_{x,x'} D_1^{(q)}(\mathcal{E}_p \circ U_{ans}(\theta)(\rho_x)\parallel \mathcal{E}_p \circ U_{ans}(\theta)(\rho_{x'})) \\&\text{(Based on Lemma~\ref{lemma:QHSD_equiv_HSD})} \\
        & \leq \sup_{x,x'} (1-p) D_1^{(q)}(U_{ans}(\theta)(\rho_x)\parallel U_{ans}(\theta)(\rho_{x'})) \\&\text{(Based on Lemma~\ref{lemma:QHSD_depolarizing})}\\
        & \leq \sup_{x,x'} (1-p) D_1^{(q)}(\rho_x\parallel \rho_{x'}) \\&\text{(Based on data-processing inequality)}\\
        &\leq (1-p) \sqrt{1-\phi}
    \end{align*}
    Thus, based on Lemma~\ref{lemma:Dobrushin}, we have the mechanism $\tilde{C} \circ A$ satisfies $(\epsilon, (1-p)\sqrt{1-\phi}\delta)$-DP.
\end{proof}

Based on the privacy guarantee of the mechanism $\tilde{C} \circ A$, we propose \NAMEB\ which is differentially private by replacing circuits \( C^{(1)}, \dots, C^{(K)} \) in \NAMEA\ by mechanisms in the form of $\tilde{C} \circ A$. The details are presented in Lemma~\ref{lemma:DPforQUPID}.

\begin{lemma}
    [$(\epsilon, \delta)$-DP for \NAMEB] Given \( A^{(1)}, \dots, A^{(K)} \) as \( (\epsilon, \delta) \)-differentially private mechanisms, \( K \) noisy quantum circuits \( \tilde{C}^{(1)}, \dots, \tilde{C}^{(K)} \), each acting on \( q \)-qubit systems, such that $\tilde{C}^{(k)}$ including the encoder $U^{(k)}_{enc}$, the parameterized ansatz $U^{(k)}_{ans}(\theta_k)$ followed by the depolarizing channel $\mathcal{E}_p$ and measurement $M^{(k)}$, and a fully-connected function $F$. \NAMEB \hspace*{1pt} can be expressed as:
    \[
\mathcal{M}_{\text{R-QUPID}} = F \circ \left[ (\tilde{C}^{(1)} \circ A^{(1)}) \oplus \dots \oplus (\tilde{C}^{(K)} \circ A^{(K)})\right]
\]

\NAMEB\ satisfies \( (K\epsilon, K(1-p)\sqrt{1-\phi} \delta) \)-DP.

    \label{lemma:DPforQUPID}
\end{lemma}

\begin{proof}
    It can be proved based on Lemma~\ref{lemma:privacy_amplification} and the parallel composition theory for non-disjoint inputs derived in~\cite{Bogatov_parallelComposition}.
\end{proof}

Finally, we can derive the certifiable robustness of \NAMEB\ as follows:

\begin{theorem}[Robustness Condition]
Given $\mathcal{M}_{\text{R-QUPID}}$, we define the scoring function $\hat{y}$ such that:
$$\hat{y}(\mathbf{x}) = softmax(\mathcal{M}_{\text{R-QUPID}}(\mathbf{x})) \forall \mathbf{x}$$
For any input \( \mathbf{x} \), if for some labels \( c \),

\[
\mathbb{E}[\hat{y}_c(x)] > e^{2K\epsilon} \max_{i \neq c} \mathbb{E}[\hat{y}_i(x)] + (1 + e^{K\epsilon})K(1-p)\sqrt{1-\phi} \delta,
\]

\noindent \NAMEB\ is robust to attacks \( \alpha \) of size \( \|\alpha\| \leq 1 \) on input \( \mathbf{x} \).
\label{theorem:robustness}
\end{theorem}
\begin{table*}[h]
    \centering
    \scalebox{0.9}{
    \begin{tabular}{cccccccccccccccccccc}
        \hline
        \multirow{2}{*}{\textbf{Metrics}} & \multirow{2}{*}{\textbf{Models}} & \multicolumn{15}{c}{\textbf{Scenarios}} \\ \cline{3-17} 
 &  & S1 & S2 & S3 & S4 & S5 & S6 & S7 & S8 & S9 & S10 & S11 & S12 & S13 & S14 & S15 \\ \hline
\multirow{5}{*}{Accuracy} & MLP & 0.854 & 0.776 & 0.849 & 0.809 & 0.782 & 0.812 & 0.829 & 0.845 & 0.820 & 0.805 & 0.824 & 0.816 & 0.850 & 0.838 & 0.807 \\ \cline{3-17} 
 & InceptionNet & 0.817 & 0.782 & 0.800 & 0.738 & 0.743 & 0.773 & 0.801 & 0.778 & 0.767 & 0.790 & 0.781 & 0.818 & 0.784 & 0.808 & 0.805 \\ \cline{3-17} 
 & MTL-LSTM & 0.869 & 0.864 & 0.875 & 0.825 & 0.864 & 0.871 & 0.864 & 0.865 & 0.840 & 0.853 & 0.871 & 0.861 & 0.885 & 0.861 & 0.852  \\ \cline{3-17} 
 & HQ-DNN & 0.823&0.796&0.824&0.794&0.771&0.795&0.782&0.761&0.815&0.800&0.808&0.798&0.783&0.808&0.789\\ \cline{3-17} 
 & FTTransformer & 0.921 & \cellcolor{cyan!50}\textbf{0.911} & 0.908 & 0.919 & \cellcolor{cyan!50}\textbf{0.892} & \cellcolor{cyan!50}\textbf{0.923} & 0.908 & \cellcolor{cyan!50}\textbf{0.932} & \cellcolor{cyan!50}\textbf{0.900} & \cellcolor{cyan!50}\textbf{0.934} & 0.912 & 0.926 & \cellcolor{cyan!50}\textbf{0.915} & 0.927 & 0.914 \\ \cline{3-17} 
 & \NAMEA \hspace*{2pt}(Ours) & \cellcolor{orange!70}\textbf{0.940} & 0.893 & \cellcolor{orange!70}\textbf{0.926} & \cellcolor{orange!70}\textbf{0.930} & 0.890 & 0.918 & \cellcolor{orange!70}\textbf{0.940} & 0.920 & 0.890 & 0.914 & \cellcolor{orange!70}\textbf{0.943} & \cellcolor{orange!70}\textbf{0.927} & 0.903 & \cellcolor{orange!70}\textbf{0.941} & \cellcolor{orange!70}\textbf{0.929} \\ \hline
\multirow{5}{*}{Precision} & MLP & 0.851 & 0.780 & 0.849 & 0.808 & 0.778 & 0.810 & 0.826 & 0.845 & 0.821 & 0.804 & 0.825 & 0.814 & 0.851 & 0.835 & 0.802 \\ \cline{3-17} 
 & InceptionNet & 0.757 & 0.740 & 0.777 & 0.703 & 0.692 & 0.719 & 0.754 & 0.747 & 0.707 & 0.738 & 0.738 & 0.785 & 0.760 & 0.753 & 0.775 \\ \cline{3-17} 
 & MTL-LSTM & 0.829 & 0.848 & 0.839 & 0.785 & 0.851 & 0.846 & 0.846 & 0.853 & 0.819 & 0.820 & 0.856 & 0.842 & 0.849 & 0.837 & 0.821 \\ \cline{3-17} 
 & HQ-DNN &0.824&0.798&0.871&0.827&0.809&0.802&0.771&0.814&0.801&0.773&0.811&0.810&0.765&0.830&0.769&\\ \cline{3-17} 
 & FTTransformer & 0.900 & \cellcolor{cyan!50}\textbf{0.888} & 0.872 & 0.907 & 0.871 & \cellcolor{cyan!50}\textbf{0.910} & 0.893 & \cellcolor{cyan!50}\textbf{0.928} & \cellcolor{cyan!50}\textbf{0.890} & \cellcolor{cyan!50}\textbf{0.921} & 0.892 & 0.899 & 0.916 & \cellcolor{cyan!50}\textbf{0.925} & 0.895 \\ \cline{3-17} 
 & \NAMEA \hspace*{2pt}(Ours) & \cellcolor{orange!70}\textbf{0.933} & 0.882 & \cellcolor{orange!70}\textbf{0.894} & \cellcolor{orange!70}\textbf{0.928} & \cellcolor{orange!70}\textbf{0.872} & 0.902 & \cellcolor{orange!70}\textbf{0.895} & 0.921 & 0.863 & 0.915 & \cellcolor{orange!70}\textbf{0.935} & \cellcolor{orange!70}\textbf{0.928} & \cellcolor{orange!70}\textbf{0.916} & 0.923 & \cellcolor{orange!70}\textbf{0.919} \\ \hline
\multirow{5}{*}{Recall} & MLP & 0.854 & 0.776 & 0.849 & 0.809 & 0.782 & 0.812 & 0.829 & 0.845 & 0.820 & 0.805 & 0.824 & 0.816 & 0.850 & 0.838 & 0.807\\ \cline{3-17} 
 & InceptionNet & 0.735 & 0.724 & 0.713 & 0.704 & 0.672 & 0.721 & 0.746 & 0.774 & 0.679 & 0.763 & 0.746 & 0.777 & 0.743 & 0.764 & 0.769 \\ \cline{3-17} 
 & MTL-LSTM & 0.815 & 0.842 & 0.853 & 0.801 & 0.838 & 0.833 & 0.833 & 0.875 & 0.825 & 0.832 & 0.846 & 0.825 & 0.869 & 0.860 & 0.820  \\ \cline{3-17} 
 & HQ-DNN &0.686&0.678&0.685&0.707&0.707&0.679&0.667&0.673&0.696&0.710&0.691&0.691&0.613&0.645&0.652&\\ \cline{3-17} 
 & FTTransformer & 0.888 & 0.886 & 0.883 & 0.908 & 0.871 & \cellcolor{cyan!50}\textbf{0.909} & 0.889 & \cellcolor{cyan!50}\textbf{0.918} & \cellcolor{cyan!50}\textbf{0.864} & \cellcolor{cyan!50}\textbf{0.921} & 0.899 & 0.898 & \cellcolor{cyan!50}\textbf{0.906} & 0.906 & 0.892 \\ \cline{3-17} 
 & \NAMEA \hspace*{2pt}(Ours) & \cellcolor{orange!70}\textbf{0.899} & \cellcolor{orange!70}\textbf{0.890} & \cellcolor{orange!70}\textbf{0.901} & \cellcolor{orange!70}\textbf{0.915} & \cellcolor{orange!70}\textbf{0.873} & 0.892 & \cellcolor{orange!70}\textbf{0.898} & 0.892 & 0.849 & 0.898 & \cellcolor{orange!70}\textbf{0.938} & \cellcolor{orange!70}\textbf{0.920} & 0.882 & \cellcolor{orange!70}\textbf{0.935} & \cellcolor{orange!70}\textbf{0.918} \\ \hline
\multirow{5}{*}{F1-Score} & MLP & 0.851 & 0.777 & 0.848 & 0.808 & 0.778 & 0.810 & 0.826 & 0.845 & 0.819 & 0.804 & 0.823 & 0.814 & 0.850 & 0.836 & 0.803\\ \cline{3-17} 
 & InceptionNet & 0.744 & 0.727 & 0.740 & 0.697 & 0.678 & 0.714 & 0.749 & 0.757 & 0.690 & 0.750 & 0.733 & 0.779 & 0.745 & 0.758 & 0.771 \\ \cline{3-17} 
 & MTL-LSTM & 0.821 & 0.844 & 0.845 & 0.792 & 0.843 & 0.839 & 0.838 & 0.864 & 0.821 & 0.826 & 0.851 & 0.833 & 0.858 & 0.848 & 0.820  \\ \cline{3-17} 
 & HQ-DNN &0.712&0.724&0.744&0.742&0.741&0.722&0.696&0.722&0.727&0.737&0.735&0.727&0.649&0.704&0.680&\\ \cline{3-17} 
 & FTTransformer & 0.894 & 0.886 & 0.877 & 0.907 & 0.871 & \cellcolor{cyan!50}\textbf{0.909} & 0.890 & \cellcolor{cyan!50}\textbf{0.923} & \cellcolor{cyan!50}\textbf{0.875} & \cellcolor{cyan!50}\textbf{0.920} & 0.895 & 0.898 & \cellcolor{cyan!50}\textbf{0.910} & 0.915 & 0.893 \\ \cline{3-17} 
 & \NAMEA \hspace*{2pt}(Ours) & \cellcolor{orange!70}\textbf{0.925} & \cellcolor{orange!70}\textbf{0.886} & \cellcolor{orange!70}\textbf{0.897} & \cellcolor{orange!70}\textbf{0.925} & \cellcolor{orange!70}\textbf{0.872} & 0.896 & \cellcolor{orange!70}\textbf{0.896} & 0.905 & 0.855 & 0.906 & \cellcolor{orange!70}\textbf{0.937} & \cellcolor{orange!70}\textbf{0.921} & 0.897 & \cellcolor{orange!70}\textbf{0.928} & \cellcolor{orange!70}\textbf{0.918} \\ \hline
\multirow{5}{*}{ROC-AUC} & MLP & 0.940 & 0.923 & 0.956 & 0.937 & 0.930 & 0.929 & 0.946 & 0.954 & 0.930 & 0.934 & 0.937 & 0.924 & 0.937 & 0.942 & 0.928 \\ \cline{3-17} 
 & InceptionNet & 0.920 & 0.896 & 0.900 & 0.903 & 0.890 & 0.918 & 0.898 & 0.907 & 0.900 & 0.918 & 0.916 & 0.915 & 0.911 & 0.906 & 0.913 \\ \cline{3-17} 
 & MTL-LSTM & 0.945 & 0.963 & 0.963 & 0.958 & 0.954 & 0.959 & 0.958 & 0.973 & 0.957 & 0.967 & 0.964 & 0.960 & 0.964 & 0.971 & 0.959  \\ \cline{3-17} 
 & HQ-DNN &0.926&0.917&0.927&0.924&0.944&0.922&0.914&0.920&0.923&0.928&0.928&0.928&0.910&0.911&0.923&\\ \cline{3-17} 
 & FTTransformer & 0.972 & \cellcolor{cyan!50}\textbf{0.972} & 0.973 & \cellcolor{cyan!50}\textbf{0.977} & 0.964 & 0.973 & \cellcolor{cyan!50}\textbf{0.972} & \cellcolor{cyan!50}\textbf{0.976} & \cellcolor{cyan!50}\textbf{0.972} & 0.978 & 0.977 & 0.976 & \cellcolor{cyan!50}\textbf{0.970} & 0.976 & \cellcolor{cyan!50}\textbf{0.976} \\ \cline{3-17} 
 & \NAMEA \hspace*{2pt}(Ours) & \cellcolor{orange!70}\textbf{0.975} & 0.966 & \cellcolor{orange!70}\textbf{0.974} & 0.973 & \cellcolor{orange!70}\textbf{0.965} & \cellcolor{orange!70}\textbf{0.979} & 0.970 & 0.964 & 0.971 & \cellcolor{orange!70}\textbf{0.983} & \cellcolor{orange!70}\textbf{0.977} & \cellcolor{orange!70}\textbf{0.980} & 0.967 & \cellcolor{orange!70}\textbf{0.978} & 0.970 \\ \hline
\multirow{5}{*}{MCC} & MLP & 0.763 & 0.691 & 0.770 & 0.737 & 0.694 & 0.725 & 0.752 & 0.789 & 0.753 & 0.723 & 0.745 & 0.750 & 0.767 & 0.760 & 0.734 \\ \cline{3-17} 
 & InceptionNet & 0.702 & 0.691 & 0.692 & 0.641 & 0.639 & 0.674 & 0.713 & 0.705 & 0.678 & 0.708 & 0.683 & 0.754 & 0.662 & 0.716 & 0.736 \\ \cline{3-17} 
 & MTL-LSTM & 0.789 & 0.810 & 0.810 & 0.762 & 0.812 & 0.812 & 0.805 & 0.817 & 0.781 & 0.792 & 0.812 & 0.811 & 0.821 & 0.799 & 0.797  \\ \cline{3-17} 
 & HQ-DNN &0.708&0.706&0.727&0.715&0.675&0.686&0.677&0.666&0.743&0.712&0.712&0.724&0.654&0.707&0.710&\\ \cline{3-17} 
 & FTTransformer & 0.874 & \cellcolor{cyan!50}\textbf{0.875} & 0.863 & 0.889 & \cellcolor{cyan!50}\textbf{0.851} & \cellcolor{cyan!50}\textbf{0.886} & 0.868 & \cellcolor{cyan!50}\textbf{0.907} & \cellcolor{cyan!50}\textbf{0.863} & \cellcolor{cyan!50}\textbf{0.906} & 0.871 & 0.901 & \cellcolor{cyan!50}\textbf{0.869} & 0.891 & 0.883 \\ \cline{3-17} 
 & \NAMEA \hspace*{2pt}(Ours) & \cellcolor{orange!70}\textbf{0.904} & 0.851 & \cellcolor{orange!70}\textbf{0.889} & \cellcolor{orange!70}\textbf{0.904} & 0.846 & 0.878 & \cellcolor{orange!70}\textbf{0.871} & 0.890 & 0.849 & 0.878 & \cellcolor{orange!70}\textbf{0.917} & \cellcolor{orange!70}\textbf{0.901} & 0.849 & \cellcolor{orange!70}\textbf{0.912} & \cellcolor{orange!70}\textbf{0.904} \\ \hline
\multirow{5}{*}{G-Mean} & MLP & 0.875 & 0.846 & 0.887 & 0.866 & 0.843 & 0.861 & 0.873 & 0.896 & 0.878 & 0.864 & 0.873 & 0.872 & 0.888 & 0.882 & 0.865 \\ \cline{3-17} 
 & InceptionNet & 0.835 & 0.829 & 0.821 & 0.813 & 0.795 & 0.828 & 0.844 & 0.858 & 0.801 & 0.853 & 0.841 & 0.863 & 0.837 & 0.855 & 0.858 \\ \cline{3-17} 
 & MTL-LSTM & 0.887 & 0.903 & 0.909 & 0.877 & 0.901 & 0.899 & 0.898 & 0.921 & 0.891 & 0.897 & 0.906 & 0.894 & 0.919 & 0.913 & 0.890  \\ \cline{3-17} 
 & HQ-DNN &0.806&0.800&0.806&0.819&0.816&0.800&0.792&0.795&0.816&0.822&0.809&0.811&0.757&0.782&0.787&\\ \cline{3-17} 
 & FTTransformer & 0.933 & \cellcolor{cyan!50}\textbf{0.932} & 0.929 & 0.944 & 0.922 & \cellcolor{cyan!50}\textbf{0.945} & 0.933 & \cellcolor{cyan!50}\textbf{0.951} & \cellcolor{cyan!50}\textbf{0.919} & \cellcolor{cyan!50}\textbf{0.952} & 0.938 & 0.940 & \cellcolor{cyan!50}\textbf{0.941} & 0.943 & 0.935 \\ \cline{3-17} 
 & \NAMEA \hspace*{2pt}(Ours) & \cellcolor{orange!70}\textbf{0.940} & 0.931 & \cellcolor{orange!70}\textbf{0.940} & \cellcolor{orange!70}\textbf{0.949} & \cellcolor{orange!70}\textbf{0.922} & 0.935 & \cellcolor{orange!70}\textbf{0.938} & 0.936 & 0.910 & 0.937 & \cellcolor{orange!70}\textbf{0.962} & \cellcolor{orange!70}\textbf{0.951} & 0.926 & \cellcolor{orange!70}\textbf{0.960} & \cellcolor{orange!70}\textbf{0.950} \\ \hline
\end{tabular}}
    \caption{Performance Metrics Across Models and Scenarios for 6-Class Classification}
    \label{tab:6-results}
\end{table*}

\vspace*{-0.25cm}
\section{Experiments}\label{sec:exp}
In this section, we provide our experimental analysis of QUPID compared to four state-of-the-art traditional deep learning methods for anomaly detection in smart grid systems. We extend our experiments to measure each models robustness capabilities when faced against two established ML attacks.
\begin{figure}[t]
    \centering\includegraphics[width=\columnwidth, trim={0 0.25cm 0.25cm 0},clip=True]{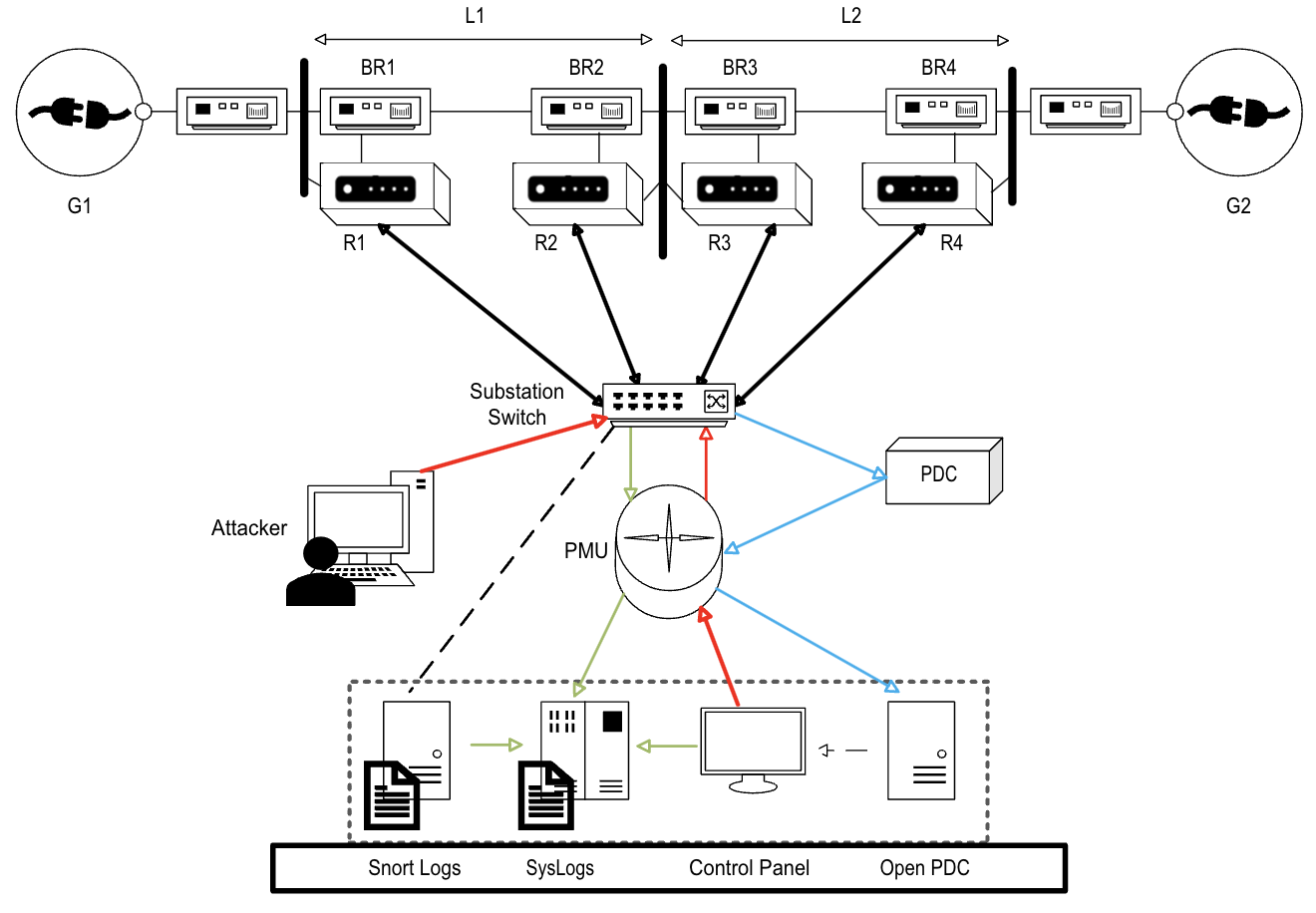}
    \caption{The smart grid system overview.}
    \label{fig:system}
\end{figure}

\noindent \textbf{Experimental Setup.} We implement our experiments with Python 3.8. Each experiment is conducted on a single GPU-assisted compute node installed with a Linux 64-bit operating system. Our testbed resources include 72 CPU cores with 377 GB of RAM in total. Our allocated node is also provisioned with 2 GPUs with 40GB of VRAM per GPU.

All quantum components of our model were implemented using 
the PennyLane QML 
simulator~\mbox{\cite{pennylane}}. To encode the complex-valued classical input data in a quantum space, we used four 5-qubit circuits (20 qubits in total). Each circuit encodes up to 32 complex-valued inputs from neighboring PMUs. The circuit encoder uses an amplitude embedding technique. Each parameterized ansatz consists of 3 alternating layers of single-qubit rotations (i.e., \texttt{RY} and \texttt{RZ} gates) and entangling layers. The entangling layers use a circular arrangement of \texttt{CNOT} gates, where each qubit $i$ is entangled with qubit $i+1$ (and the last with the first). The rotational parameters in these layers constitute the trainable parameters. In addition, to evaluate the robustness of R-QUPID, we explicitly simulated the effects of quantum noise. As described in our theoretical analysis , we introduced a depolarizing channel after the parameterized ansatz in each circuit.
\\ 
\noindent \textbf{Dataset.} Created by Oak Ridge National Laboratory \cite{Morris_ICSDataset}, the ICS dataset \footnote{\url{https://sites.google.com/a/uah.edu/tommy-morris-uah/ics-data-sets}} 
is based on a small-scale smart grid system displayed in Figure~\ref{fig:system}. Specifically, the dataset reports simulated results in a smart grid system 
composed of two generators (G1, G2), four circuit breakers (BR1-BR4), and four intelligent electronic devices (IEDs) (R1-R4). Two transmission lines (L1, L2) connect the breakers while a phasor measurement unit (PMU) processes and stores the data. The IEDs are relays that utilize a distance protection scheme to trigger breakers when faults are detected. The dataset is composed of 15 scenarios that we denote as S1-S15 in Table~\ref{tab:6-results}. Each scenario contains 37 different potential events 
consisting of normal operations, natural events, and cyber threats. Specifically, a normal operation is line maintenance where one or more relays are purposely disabled for servicing. An example of a naturally occurring event is a short-circuit fault where a power line is shorted at multiple locations and specified by percentage distances along a line. Cyber threats include remote tripping command injection, altering relay settings, and data injections.

There are a total of 42 classes and we have subcategorized them into 6 classes, labeled 0-5 for classification. 
Each scenario is comprised of 128 features, with gathered data from 4 PMUs, snort logs, and system logs. The PMUs measure data for 29 variables while the control panel logs an additional 12 features. During data exploration, we removed 8 features: 4 control log and 4 snort log entries, because they exhibited zero variance across all samples. Plus, they did not affect model training. Input IDs are also removed as they are directly linked to the sample labels. Therefore, we only use 119 total features. We scaled the data using the MinMaxScaler function and partition the data with a 70/30 split before proceeding with training.\\ 
\noindent \textbf{Baselines.} We compare QUPID to four state-of-the-art supervised deep learning methods. For each approach, we employ the same hyperparameters as in their respective papers. 
We briefly explain each below:\\ 
\vspace*{-0.25cm}
\begin{itemize}
    \item \textbf{Multi-Layer Perceptron (MLP)} - We implement the same MLP used in \cite{han2022adbench}. 
    As in \cite{han2022adbench}, we use the default hyperparameters from the \texttt{Sci-Kit Learn} library \footnote{\url{https://scikit-learn.org/stable/}}.
    \item \textbf{Inception v3 Network (InceptionNet)} - We implement the InceptionNet model from \cite{pan2024multi}. 
    It consists of an initial convolutional layer, multiple Inception modules, a global average pooling layer, and a 
    dense layer to output classification predictions. The Inception modules use 1 $\times$ 1 convolutional kernels to reduce computational cost.
    \item \textbf{Multi-Task Learning-Based Long-Short Term Memory (MTL-LSTM)} - Following \cite{ganjkhani2023multi}, we implement the MTL-LSTM with an LSTM layer, two fully connected layers, and one softmax layer.
    \item \textbf{Quantum-Classical Hybrid
Deep Neural Network (QH-DNN)} - This quantum architecture represents a state-of-the-art and widely adopted approach for anomaly detection~\mbox{\cite{Wang2023,Bhowmik2024}}. It consists of two main phases: (1) feature extraction using classical deep learning to generate real-valued representations from the input data, and (2) anomaly detection using a variational quantum circuit. For feature extraction, we employ two fully connected layers, while the quantum component utilizes \texttt{RY} gates for rotation and \texttt{CNOT} gates for entanglement.
    \item \textbf{Feature Tokenizer + Transformer (FTTransformer)} - As described in \cite{han2022adbench}, the FTTransformer \cite{gorishniy2021revisiting} is a sophisticated adaption of the Transformer architecture originally presented in \cite{vaswani2017attention} for tabular data.
\end{itemize}
\noindent \textbf{Metrics.} We employ standard performance metrics to assess the efficacy of QUPID compared to our chosen baselines. These metrics include:
\begin{itemize}
    \item \textbf{Accuracy} - Describes how many predictions were correct, displaying overall prediction performance.
    \item \textbf{Precision} - Describes the proportion of correct positive predictions out of all positive predictions by a model.
    \item \textbf{Recall} - Describes the proportion of actual positive samples that were correctly predicted by a model.
    \item \textbf{F1-Score} - Describes the predictive capabilities of a model based on its per-class performance. It is considered the harmonic mean of precision and recall.
\end{itemize}
Anomaly detection datasets are highly imbalanced as anomalies are rarities. For this reason, we also opt to use ROC-AUC, Matthew's Correlation Coefficient (MCC), and Geometric Mean (G-Mean) to give a comprehensive overview of each models performance regardless of class imbalance. 
\begin{itemize}
    \item \textbf{ROC-AUC} - Evaluates how well a model can distinguish between majority and minority classes of a dataset.
    \item \textbf{MCC} - Considers true positives, true negatives, false positives, and false negatives when evaluating a classifier, making it a well-balanced metric.
    \item \textbf{G-Mean} - Evaluates the square root of the product of the true positive rate and true negative rate to give a balanced perspective of model performance regardless of class imbalance.
\end{itemize}

\noindent \textbf{Attacks.} For our robustness experiments we opt to use 2 well-established attacks. Each is described below:
\begin{itemize}
    \item \textbf{Fast Gradient Sign Method (FGSM)} - Generates 
    adversarial examples by adding subtle perturbations to the input data based on a model's gradients to cause misclassifications while remaining imperceptible \cite{goodfellow2014explaining}.
    \item \textbf{Projected Gradient Descent (PGD)} - Iteratively 
    generates adversarial examples via gradient steps followed by projecting the perturbations back to the original space to remain imperceptible. It aims to maximize model error while restricting the perturbations to a defined range \cite{madry2018towards}.
\end{itemize}

\subsection{Comparison of Traditional and QML Models}
Table~\ref{tab:6-results} illustrates the performance of QUPID compared to our baseline methods across 15 scenarios and 7 metrics in the 6-class classification configuration. QUPID was employed with the Adam optimizer, the ReLU activation function, a learning rate of $10^{-2}$, and a batch size of 128. Each of our experiments were run for 200 epochs. 

We observe that QUPID outperforms every baseline across all metrics for the \textit{majority} of scenarios (e.g., at least 8 out of 15 scenarios). In particular, QUPID dominates all methods and metrics on S1, S3, S11, and S12. This suggests that QUPID is more effective under certain data distributions compared to other methods. Even more interesting is that QUPID's most impressive metrics are recall and F1-score where both metrics are their highest for 10 out of 15 scenarios. This suggests that QUPID excels at identifying positive instances, an important capability in anomaly detection scenarios, while consistently maintaining a balance between the trade-off of precision and recall. The high values of ROC-AUC, MCC, and G-Mean further emphasize QUPID's robustness to class imbalance. In majority of cases, QUPID outperforms the FTTransformer, but in cases where the FTTransformer surpasses QUPID could be further examined to uncover areas of refinement. The quantum baseline, HQ-DNN, 
achieves comparable performance to the MLP, InceptionNet, and MTL-LSTM, but falls short of the performance achieved by FTTransformer and QUPID. This highlights the importance of incorporating complex-valued representations to achieve higher performance. Furthermore, QUPID consistently outperforms the MTL-LSTM despite its sequential processing capabilities. Another key point to observe are the fluctuations in metrics for the baseline models. QUPID's results are much more consistent across all scenarios and metrics, suggesting high generalizability and reliability. The inference time for QUPID and all baseline models is negligible (less than $0.1$ seconds per input instance), so they all are suitable for practical deployment.

\begin{figure*}[htp!]
    \centering\includegraphics[width=\textwidth, trim={0 0.05cm 0.25cm 0.25cm},clip=True]{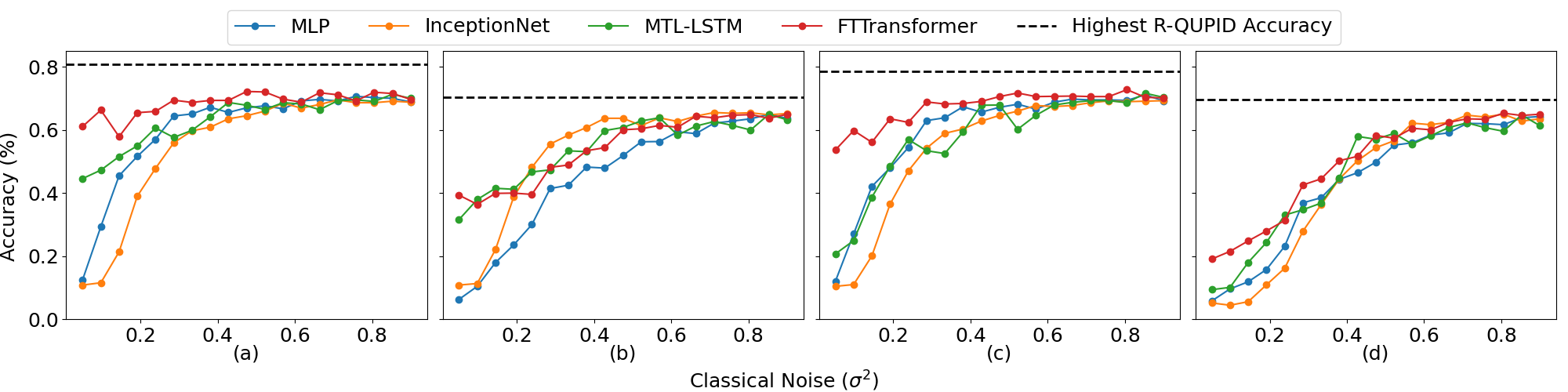}
    \caption{Performance under attack with varying attack bounds. (a) and (b) refer to the FGSM attack while (c) and (d) refer to the PGD attack, configured with 0.05 and 0.1 attack bounds, respectively.}
    \label{fig:attacks}
\end{figure*}
\vspace*{-0.25cm}
\subsection{Robustness of QML Compared to Baselines}
We evaluate a range of classical noise values for our classical DL baselines and compared them with the most optimal performance of R-QUPID, which combines quantum noise and classical noise. This comparison is necessary because the measurements of quantum noise and classical noise are not equivalent. Classical noise follows a Gaussian mechanism quantified by $\sigma^2$, while quantum noise arises from a depolarizing channel quantified by $p$. The values $\sigma^2$ and $p$ represent the noise levels of each model, respectively. 

Figure~\ref{fig:attacks} illustrates that R-QUPID outperforms all baseline methods. The baselines are evaluated with classical noise values ranging from 0.05 to 0.9. Each plot displays model performance under the FGSM and PGD attacks, with attack bounds set to 0.05 and 0.1. Particularly, for the FGSM attacks, R-QUPID reaches $\approx$ 81\% and $\approx$ 70\%, respectively. In the first case, the baselines perform $\approx$ 10\% less than that of R-QUPID and $\approx$ 5\% less in the second case. For the PGD attack, R-QUPID reaches $\approx$ 79\% and $\approx$ 70\%, respectively. In the first case, the baselines perform $\approx$ 8\% less than that of R-QUPID and $\approx$ 3\% in the second case. A key observation from these results is that although increasing classical noise improves model robustness, the baselines struggle to reach R-QUPID’s performance, even at higher noise levels. This suggests that R-QUPID’s resilience under adversarial conditions is not solely credited to a byproduct of noise, but rather to an inherent advantage of its quantum-inspired design. Furthermore, we observe a connection between quantum noise and classical noise that contributes positively to R-QUPID’s robustness, allowing it to maintain higher accuracy even under varying attack intensities. This indicates that quantum-informed representations may provide a more stable feature space, reducing the effectiveness of adversarial perturbations.
\vspace*{-0.25cm}
\section{Conclusion}\label{sec:con}
In this work we have presented QUPID as a partitioned quantum neural network for anomaly detection in smart grid systems. We extend our network to R-QUPID to provide certifiable robustness guarantees when under attack. Our findings emphasize the advantages of QUPID and R-QUPID in smart grid anomaly detection as they consistently outperform state-of-the-art traditional deep learning approaches across many datasets and metrics. Fundamentally, we have shown that quantum-enhanced feature representations facilitate more effective modeling of the high-dimensional and nonlinear characteristics seen in smart grid data. Our results also emphasize the resilience of QML models compared to traditional ML models when faced with adversarial threats. Furthermore, the partitioning scheme mitigates the scalability issue typically seen in quantum computing, making it an effective solution for large-scale utilization. With regard to the rise of cyber threats in smart grid environments, this work provides a new research avenue toward robust and scalable anomaly detection solutions while further enforcing the growing importance of QML in securing modern smart grid systems.
\vspace*{-0.25cm}
\section*{Acknowledgement}
This work was supported by the Korea Institute of Energy Technology Evaluation and Planning (KETEP) grant funded by the Korea government (MOTIE) (RS-2023-00303559, A Study on Development of Cyber-Physical Attack Response System and Security Management System for Maximizing Availability of Real-Time Distributed Resources).


\bibliographystyle{IEEEtran}
\bibliography{ref}

\clearpage
\end{document}